\documentclass{article}

\newif\ifarXiv
\arXivtrue

\ifarXiv
  \usepackage[preprint]{neurips_2021}
\else
  \usepackage[]{neurips_2021}
\fi

\usepackage[utf8]{inputenc} 
\usepackage[T1]{fontenc}    
\usepackage{url}            
\usepackage{booktabs}       
\usepackage{amsfonts}       
\usepackage{nicefrac}       
\usepackage{microtype}      
\usepackage{xcolor}         

\newif\ifmydraft
\mydraftfalse

\usepackage{natbib}
\usepackage[hidelinks]{hyperref}

\ifmydraft
  \newcommand{\draftcolor}{purple}
  \newcommand{\draftcolorpage}{black!20}

\else
  \newcommand{\draftcolor}{black}
\fi

\usepackage[utf8]{inputenc} 
\usepackage[T1]{fontenc}    
\usepackage{amssymb,amsfonts, amsmath, amsthm}
\ifmydraft\else
\usepackage{microtype}      
\fi
\usepackage{graphicx}
\usepackage{xcolor}
\ifmydraft\usepackage{showlabels}\else\fi
\usepackage{comment}
\usepackage{caption}
\usepackage{subcaption}
\usepackage{tikz}
\usepackage{tikz, pgfplots, booktabs}
  \usetikzlibrary{arrows.meta}
   \usetikzlibrary{positioning}
   \tikzset{
     included node/.style={circle, draw=black!100, thick, on grid, minimum width=13mm, fill=black!20}, 
    hidden node/.style={circle, draw=white, thick, on grid, minimum width=0.2cm},
    included connection/.style={->, thick, draw=black!100},
    hidden connection/.style={->, thick, draw=black!30, draw opacity=0},
    fused connection/.style={->, thick, draw=black!100},
    fidden connection/.style={->, thick, draw=black!30, draw opacity=0},
    node title/.style={above=0.8cm of five-one, font=\bfseries},
    general/.style={node distance=1cm and 0.7cm}
}

\usetikzlibrary{
  arrows.meta,
  positioning,
	calc,
	decorations.pathreplacing,
	positioning,
	shapes.misc,
	tikzmark,
}

\ifmydraft
  \pagecolor{\draftcolorpage}
  \newcommand{\jl}[1]{\color{blue}#1\color{black}}
\else
  \newcommand{\jl}[1]{}
\fi

\theoremstyle{definition}

\newtheorem{theorem}{Theorem}

\makeatletter 
\newcommand*{\deq}{\ensuremath{\mathrel{\rlap{%
\raisebox{0.3ex}{$\m@th\cdot$}}%
\raisebox{-0.3ex}{$\m@th\cdot$}}=}}
\makeatother

\newcommand{\abs}[1]{\lvert#1\rvert}
\newcommand{\absB}[1]{\bigl\lvert#1\bigr\rvert}
\newcommand{\absBB}[1]{\Bigl\lvert#1\Bigr\rvert}
\newcommand{\absBBB}[1]{\biggl\lvert#1\biggr\rvert}
\newcommand{\absBBBB}[1]{\Biggl\lvert#1\Biggr\rvert}

\newcommand{\nbrlayers}{\textcolor{\draftcolor}{\ensuremath{l}}}

\newcommand{\datainE}{\textcolor{\draftcolor}{\ensuremath{x}}}

\newcommand{\loss}{\ensuremath{\textcolor{\draftcolor}{\ensuremath{\mathfrak{h}}}}}
\newcommand{\lossF}[1]{\textcolor{\draftcolor}{\ensuremath{\loss[#1]}}}
\newcommand{\lossFB}[1]{\textcolor{\draftcolor}{\ensuremath{\loss\bigl[#1\bigr]}}}

\newcommand{\function}{\textcolor{\draftcolor}{\ensuremath{\mathfrak f}}}

\newcommand{\functionF}[1]{\textcolor{\draftcolor}{\function[#1]}}

\newcommand{\functionFB}[1]{\textcolor{\draftcolor}{\function\bigl[#1\bigr]}}
\newcommand{\functionFBB}[1]{\textcolor{\draftcolor}{\function\Bigl[#1\Bigr]}}

\newcommand{\R}{\textcolor{\draftcolor}{\mathbb{R}}}

\newcommand{\scale}{\textcolor{\draftcolor}{\ensuremath{\alpha}}}
\newcommand{\argumentE}{\textcolor{\draftcolor}{\ensuremath{z}}}

\newcommand{\functionscale}{\textcolor{\draftcolor}{\ensuremath{\function_{\scale}}}}
\newcommand{\functionscaleF}[1]{\textcolor{\draftcolor}{\functionscale[#1]}}
\newcommand{\functionscaleFB}[1]{\textcolor{\draftcolor}{\functionscale\bigl[#1\bigr]}}
\newcommand{\functionscaleFBB}[1]{\textcolor{\draftcolor}{\functionscale\Bigl[#1\Bigr]}}

\newcommand{\functiontanhscale}{\textcolor{\draftcolor}{\ensuremath{\function_{\operatorname{tanh},\scale}}}}
\newcommand{\functiontanhscaleF}[1]{\textcolor{\draftcolor}{\functiontanhscale[#1]}}

\newcommand{\namelog}{\textcolor{\draftcolor}{\texttt{logsigmoid}}}
\newcommand{\nametanh}{\textcolor{\draftcolor}{\texttt{tanh}}}
\newcommand{\namerelu}{\textcolor{\draftcolor}{\texttt{relu}}}

\newcommand{\namesoft}{\textcolor{\draftcolor}{\texttt{softsign}}}
\newcommand{\namearctan}{\textcolor{\draftcolor}{\texttt{arctan}}}

\newcommand{\functiontanh}{\textcolor{\draftcolor}{\ensuremath{\function_{\operatorname{tanh}}}}}
\newcommand{\functiontanhF}[1]{\textcolor{\draftcolor}{\functiontanh[#1]}}
\newcommand{\functiontanhFB}[1]{\textcolor{\draftcolor}{\functiontanh\bigl[#1\bigr]}}
\newcommand{\functiontanhFBB}[1]{\textcolor{\draftcolor}{\functiontanh\Bigl[#1\Bigr]}}

\newcommand{\neuronparE}{\textcolor{\draftcolor}{\ensuremath{\theta}}}

\usepackage[textwidth=3.5cm, textsize=small]{todonotes}

\newcommand{\datainEnew}{\textcolor{\draftcolor}{\ensuremath{x_{\operatorname{sgd}}}}}
\newcommand{\dataoutEnew}{\textcolor{\draftcolor}{\ensuremath{y_{\operatorname{sgd}}}}}

\newcommand{\tj}[1]{\tag*{\footnotesize #1}}

\newcommand{\levellpr}{\textcolor{\draftcolor}{\ensuremath{t}}}

\newcommand{\tuningparameter}{\textcolor{\draftcolor}{\ensuremath{r}}}

\newcommand{\deffit}{\textcolor{\draftcolor}{\ensuremath{d_{(\dataoutEnew,\datainEnew)}}}}
\newcommand{\deffitA}{\textcolor{\draftcolor}{\ensuremath{\abs{\deffit}}}}

\newcommand{\lossW}{\textcolor{\draftcolor}{\ensuremath{\lossFB{\dataoutEnew-\mathfrak{g}_{(\neuronparE^1,\dots,\neuronparE^{\nbrlayers})}[\datainEnew]}}}}

\title{Regularization and Reparameterization\\
	Avoid Vanishing Gradients in Sigmoid-Type Networks}

\author{%
	Leni Ven \\
        University of Waterloo\\
	\texttt{shwen@uwaterloo.ca} \\
	\And
	Johannes Lederer \\
        Ruhr-University Bochum\\
	\texttt{johannes.lederer@rub.de} \\
}

\ifmydraft
   \newcommand{\unc}[1]{\color{orange}Uncritical: #1\color{black}}
\else
    \newcommand{\unc}[1]{}
\fi

\begin{document}

\maketitle

\begin{abstract}
	Deep learning requires several design choices,
	such as the nodes' activation functions and the widths, types, and arrangements of the layers.
	One consideration when making these choices is the vanishing-gradient problem,
	which is the phenomenon of algorithms getting stuck  at suboptimal points due to small gradients.
	In this paper, we revisit the vanishing-gradient problem in the context of sigmoid-type activation.
	We use mathematical arguments to highlight two different sources of the phenomenon, namely large individual parameters and effects across layers,
	and to illustrate two simple remedies,
	namely regularization and rescaling.
	We then demonstrate the effectiveness of the two remedies in practice.
	In view of the vanishing-gradient problem being a main reason why tanh and other sigmoid-type activation has become much less popular than relu-type activation,
	our results bring sigmoid-type activation back to the table.
\end{abstract}

\section{Introduction}
It is well known that there can be suboptimal regions in the optimization landscapes of deep learning where the gradients are too small for gradient-based algorithms to make reasonable progress.
In other words,
gradient-based algorithms, such as stochastic-gradient descent~\citep{bottou2010large}, can be stuck in unfavorable parts of the parameter space.
This phenomenon is called the \emph{vanishing-gradient problem} \citep{Hochreiter1991,Hochreiter2001}.
There are several different approaches to avoiding the vanishing-gradient problem, 
including special weight initializations \citep{Mishkin2015},
second-order algorithms~\citep{Martens2010},
and layerwise updates \citep{Vincent2008}.
But the most popular approach is to replace sigmoid-activation functions
such as \nametanh, \namelog, and \namearctan\ by piecewise-linear activation functions such as \namerelu~\citep{Lederer2021}.

This, of course, raises the question of whether the vanishing-gradient problem makes sigmoid-type activation a thing of the past altogether.

In this paper,
we revisit this question from both mathematical and  empirical perspectives.
Using a simple toy model,
we argue that the vanishing-gradient problem in sigmoid networks can have two different sources:
large parameters and joint effects among many layers.
We then show that both sources have a very simple remedy:
vanishing gradients in the large-parameters regime can be avoided by using standard \emph{$\ell_2$-regularization},
and vanishing gradients in the many-layers regime can be avoided by using a \emph{reparametrization}.
We can think of these remedies as simple but mathematically justified analogs of reparametrization and batch normalization, respectively.
We then show empirically that these two remedies are indeed effective in practice.

In summary, we make three main contributions:
\begin{itemize}
\item We highlight that the vanishing-gradient problem has different sources,
and that each of these sources should be addressed differently.
\item We introduce two very simple techniques for avoiding vanishing gradients in sigmoid-type networks.
\item More generally, we argue that  sigmoid-type activation is still of high interest in deep learning.
\end{itemize}

\paragraph{Outline}
Section~\ref{sec:math} contains mathematical insights into the vanishing gradient problem and our proposed remedies.
Section~\ref{sec:empirical} supports our claims empirically.
Section~\ref{sec:discussion} summarizes our main conclusions.
The Appendix contains some experimental details.

\section{The Vanishing-Gradient Problem: Sources and Remedies}
\label{sec:math}
We first discuss the vanishing-gradient problem from a mathematical perspective and then point out two remedies:
regularization and rescaling.

\subsection{Small gradients can have two sources}
\label{sec:sources}
Let us first discuss the basics of the vanishing-gradient problem.
The vanishing-gradient problem is often mentioned in the context of recurrent neural networks~\citep{Tan2019},
but it turns out that a much simpler feedforward model is sufficient to illustrate the problem---as well as potential remedies later.
To be specific,
we consider a network of~$\nbrlayers$ concatenated \nametanh-neurons with one parameter each, 
that is, 
we consider the functions
  \begin{align*}
\R~&\to~\R\\
  \datainE~&\mapsto~ \mathfrak{g}_{(\neuronparE^1,\dots,\neuronparE^l)}[\datainE]\,\deq\,\functiontanhFBB{\neuronparE^{\nbrlayers}\functiontanhFB{\neuronparE^{\nbrlayers-1}\cdots\functiontanhF{\neuronparE^1 \datainE}}}
  \end{align*}
parametrized by $\neuronparE^1,\dots,\neuronparE^{\nbrlayers}\in\R$.
Recall that the \nametanh\ function is defined as
  \begin{align*}
  \functiontanh~:~\R~&\to~(-1,1)\,;\\
\argumentE~&\mapsto~\functiontanhF{\argumentE}\,\deq\,\frac{e^{\argumentE}-e^{-\argumentE}}{e^{\argumentE}+e^{-\argumentE}}\,.
  \end{align*}

Given a network model and a (differentiable) loss function $\loss\,:\,\R\to\R$ that operates on the differences between predicted and actual outputs, 
the parameters are fitted to data by using a training algorithm.
Stochastic-gradient descent, 
one of the most popular training algorithms, updates the parameters sequentially.
In its simplest form,
each update of stochastic-gradient descent is a step in the direction of the gradient of $\lossF{\dataoutEnew-\mathfrak{g}_{(\neuronparE^1,\dots,\neuronparE^{\nbrlayers})}[\datainEnew]}$
with respect to the parameters $\neuronparE^1,\dots,\neuronparE^{\nbrlayers}$,
where $(\dataoutEnew,\datainEnew)\in \R\times \R$ is a fixed output-input pair.
If the gradient is very small in absolute value (or even equal to zero),
the gradient update has almost no impact on the parameters,
that is, 
there is no progress in learning the parameters.
If this happens at unsatisfactory parameter values and repeatedly for several output-input pairs,
we speak of a \emph{vanishing-gradient problem.}

We argue that the vanishing-gradient problem has two different sources.
To identify these sources,
we first introduce the shorthand
\begin{equation*}
  \function^i~\deq~\functiontanhFBB{\neuronparE^{i}\functiontanhFB{\neuronparE^{i-1}\cdots\functiontanhF{\neuronparE^1 \datainEnew}}}
\end{equation*}
for all $i\in\{1,\dots,\nbrlayers\}$ and $\function^0\deq1$.
Observe that $\function^i\in(-1,1)$ by definition of \nametanh.
Using the shorthand,
the partial derivative of the  loss with respect to $\neuronparE^k$ (that is, the $k$th element of the gradient) can be written as
\begin{equation*}
 \frac{\partial}{\partial\neuronparE^k}\lossW~=~\deffit\,\function^{k-1}\bigl(1-(\function^{\nbrlayers})^2\bigr)\cdots\bigl(1-(\function^k)^2\bigr)\,\prod^{\nbrlayers}_{i=k+1}\neuronparE^i\,,
\end{equation*}
where $\deffit\deq\loss'[\dataoutEnew-\mathfrak{g}_{(\neuronparE^1,\dots,\neuronparE^{\nbrlayers})}[\datainEnew]]$,
and where we have used the chain rule and the well-known fact $\nametanh'=1-\nametanh^2$ (and, as usual, we set $\prod^{\nbrlayers}_{i=\nbrlayers+1}\neuronparE^i\deq 1$ to make the right-hand side of the above equality well-defined in the case $k=\nbrlayers$).
A typical  loss function is the least-squares loss $\lossF{\argumentE}=\argumentE^2$,
where $\deffit=2(\dataoutEnew-\mathfrak{g}_{(\neuronparE^1,\dots,\neuronparE^{\nbrlayers})}[\datainEnew])$.
The partial derivative can have a small magnitute despite an unsatisfactory model fit, that is, despite large $\deffitA$,
for two main reasons.
The first reason is that one of the factors $1-(\function^{\nbrlayers})^2,\dots, 1-(\function^k)^2$ can be very small in absolute value,
which is the case when one of the parameters $\neuronparE^{k},\dots,\neuronparE^{\nbrlayers}$ is large (recall that $\abs{\functiontanhF{\argumentE}}\to1$ for $\argumentE\to\pm\infty$).
The other factors cannot balance the small factor in view of $\abs{\function^{k-1}},\abs{1-(\function^{\nbrlayers})^2},\dots, \abs{1-(\function^k)^2}<1$
and $(1-(\function^k[v\argumentE])^2)\cdot \argumentE\to 0$ for $\argumentE\to\infty$ and any fixed $v\in\mathbb{R}$
(see Theorem~\ref{res:lpr} below for a formal version of this statement).
For further reference, we call this source of small gradients the \emph{large-parameters regime.}

But large parameters are not a necessary condition for small gradients.
Consider, for example, $\datainEnew,\neuronparE^1,\dots,\neuronparE^{\nbrlayers}=1.5$.
Since $(1-(\nametanh[1.5])^2\cdot 1.5\leq 1/2$,
it then holds that (see Theorem~\ref{res:MLR} below for a formal version of this statement)
\begin{equation*}
 \frac{\partial}{\partial\neuronparE^k}\lossW
~\leq~ \deffitA\, \frac{1}{2^{\nbrlayers-k+1}}\,,
\end{equation*}
which can be very small for large $\nbrlayers-k$, that is, at the inner layers of deep networks.
Hence, gradients can be small as a result of a joint effect of many layers.
For further reference,
we call this source of small gradients the \emph{many-layers regime.}

The simple \nametanh-network is convenient to work out the two sources of small gradients,
but the observations can be transferred, of course, to more intricate architectures and other sigmoid-type activations (we will elaborate on this aspect later).
It is also important to note that the vanishing-gradient problem is not about how the gradients are computed (using backpropagation, for example) but about the gradients themselves.
And finally, vanishing gradients are not associated with neural networks per se but with using gradients for training neural networks.

\subsection{A remedy for the large-parameters regime}
\label{sec:remedylpr}
We now show that regularization can avoid the vanishing-gradient problem in the large-parameters regime.
It is again sufficient to focus on the toy example introduced above.
One can readily show the fact that  the $k$th partial derivative of the loss complemented with weight decay is 
\begin{multline*}
 \frac{\partial}{\partial\neuronparE^k}\Bigl(\lossW+\frac{\tuningparameter}{2}\sum_{i=1}^{\nbrlayers} (\neuronparE_i)^2\Bigr)\\=~
\deffit
\,\function^{k-1}\bigl(1-(\function^{\nbrlayers})^2\bigr)\cdots\bigl(1-(\function^k)^2\bigr)\,\prod^{\nbrlayers}_{i=k+1}\neuronparE^i~~+~~\tuningparameter\neuronparE_k\,,
\end{multline*}
where $\tuningparameter\in[0,\infty)$ is the tuning parameter of the regularizer.
The general idea is simple:
the first term vanishes for large parameters,
but the second term---which is due to the regularization---is large for large parameters.
The following theorem formalizes this observation.
\begin{theorem}[Large-Parameters Regime]
\label{res:lpr}
Consider a fixed but arbitrary constant $\levellpr\in(0,1)$.
Assume that  $\tuningparameter\geq \deffitA$ and
\begin{equation*}
\abs{\neuronparE^i} ~\geq~ \frac{\log[1/\levellpr]}{\min\{\abs{\datainEnew},1/2\}}~~\text{for all}~i\in\{1,\dots,\nbrlayers\}\,.
\end{equation*}
Then,
it holds for all $k\in\{1,\dots,\nbrlayers\}$ that
\begin{equation*}
\absBB{\frac{\partial}{\partial\neuronparE^k}\lossW}~\leq~ \deffitA\,\levellpr\,,
\end{equation*}
while
\begin{equation*}
\absBB{\frac{\partial}{\partial\neuronparE^k}\Bigl(\lossW+\frac{\tuningparameter}{2}\sum_{i=1}^{\nbrlayers} (\neuronparE_i)^2\Bigr)}~\geq~ \deffitA\,\log[1/\levellpr]\,.
\end{equation*}
\end{theorem}
\noindent 
Importantly, $\deffitA\levellpr\to 0$ for $\levellpr\to0$ while $\deffitA\log[1/\levellpr]\to\infty$ for $\levellpr\to0$,
which leads to the following conclusion:
Assume that the network does not fit the data point well, that is, $\deffitA\gtrsim 1$.
In this case, the gradient should not be too small,
so that the optimization can make progress toward a better fit.
But applying  the theorem with $\levellpr\approx 0$ shows that the gradient is small in the large-parameters regime unless regularization with a sufficiently large  tuning parameter is applied.
Hence,
the theorem illustrates once more that deep learning can suffer from vanishing gradients in the large-parameters regime,
and it confirms that regularization can avoid this problem.

Theorem~\ref{res:lpr} above and Theorem~\ref{res:MLR} below are based on the following four properties of the \nametanh\ function:
\begin{enumerate}
\item Boundedness: $\abs{\functiontanhF{\argumentE}}\leq 1$  for all $\argumentE\in\R$.
\item Rapidly vanishing derivatives: $\lim\limits_{\argumentE\to\pm\infty}(\functiontanh)'[\argumentE]\cdot\argumentE=0$.
\item Strictly positive limit:  $\lim\limits_{\argumentE\to\infty}\functiontanhF{\argumentE}> 0$.
\item Linearity at the origin: $\lim\limits_{\scale\to\infty}\scale\functiontanhF{\argumentE/\scale}=a_{\operatorname{lin}}\argumentE$ for some $a_{\operatorname{lin}}\in(0,\infty)$ and  for all $\argumentE\in\R$.
\end{enumerate}
One can verify readily that these four conditions are also met by  other standard sigmoid-type activation functions,
such as \namelog\ and \namearctan;
hence,
it is straightforward to transfer our proofs and results for the  \nametanh\ function to other sigmoid-type activation functions.

We also point out the fact that the assumption $\tuningparameter\geq \deffitA$ also makes sense from a statistical perspective;
indeed, 
it is well known that the tuning parameters of regularizers should scale with the differences between predicted and actual outputs.
We refer to \citet{huang2021tuning} and \citet{Taheri2020} for some insights on this topic.

We finally argue that regularization, in our context, can be interpreted as an alternative to  initialization schemes. 
The idea of initialization schemes,
such as those in \citet{Glorot2010},
is to restrict the parameter training to a benign basin by choosing special, random, initial values for the parameters. 
Regularization works in a similar way in the sense that it also pushes the training toward more favorable parts of the parameter space.
But there are three main differences:
First, 
initialization is only applied once,
which means that there is no control during training;
regularization, in contrast,
controls the gradients throughout the training process.
Second,
initialization tries to focus training on some basins,
that is, it essentially reduces the parameter space,
while regularization allows the algorithm to explore the entire parameter space if needed.
Third,
 initialization schemes for  sigmoid-type activation are  not equipped with any mathematically justification;
for example,
\citet{Glorot2010} make the highly questionable assumption of a linear approximation,
\citet{He2015} specialize on \namerelu-type activation,
and \citet{Mishkin2015} neglect mathematical aspects altogether (moreover, their goal of unit empirical variance is not suitable for sigmoid-type activation in the first place).
Hence,
regularization is related to initialization but provides a continuous,
more flexible,
and mathematically justified control of the gradients.

\begin{proof}[Proof of Theorem~\ref{res:lpr}]
Observe first that $\abs{\function^{k-1}}\leq 1$ by the definition of \nametanh.

Then, note that
\begin{equation*}
   \abs{\function^1} ~=~\abs{\functiontanhF{\neuronparE^{1}\datainEnew}}~\geq~\functiontanhF{1}~\geq~\frac{1}{2}\text{~~~~~and~~~~~}   \abs{\function^i} ~=~\abs{\functiontanhF{\neuronparE^{i}\function^{i-1}}}~~\text{for all}~i\in\{2,\dots,\nbrlayers\}\,.
\end{equation*}
By induction (recall that $\abs{\neuronparE^2},\dots,\abs{\neuronparE^{\nbrlayers}}\geq 2$),
we conclude that $\abs{\function^i}\geq 1/2$ for all $i\in\{0,\dots,\nbrlayers\}$ and, therefore,
\begin{equation*}
  \abs{\neuronparE^{k}} ~\geq~\frac{\log[1/\levellpr]}{2\abs{\function^{k-1}}}\,.
\end{equation*}
We then find
  \begin{align*}
    \absB{1-(\function^{k})^2}~&=~\absBBB{1-\biggl(\frac{e^{\neuronparE^{k} \function^{k-1}}-e^{-\neuronparE^{k} \function^{k}}}{e^{\neuronparE^{k} \function^{k-1}}+e^{-\neuronparE^{k} \function^{k-1}}}\biggr)^2}\tj{by definition of $\function^{k}$ and \nametanh}\\
&=~\absBBB{\frac{(e^{2\neuronparE^{k} \function^{k-1}}+e^{-2\neuronparE^{k} \function^{k-1}}+2)-(e^{2\neuronparE^{k} \function^{k-1}}+e^{-2\neuronparE^{k} \function^{k-1}}-2)}{(e^{\neuronparE^{k} \function^{k-1}}+e^{-\neuronparE^{k} \function^{k-1}})^2}}\tj{summarizing the two terms}\\
&=~\frac{4}{(e^{\neuronparE^{k} \function^{k-1}}+e^{-\neuronparE^{k} \function^{k-1}})^2}\tj{consolidating}\\
&\leq~\frac{4}{e^{2\abs{\neuronparE^{k} \function^{k-1}}}}\tj{$e^a+e^{-a}\geq e^{\abs{a}}$}\\
&\leq~\levellpr\tj{penultimate display}\,.
  \end{align*}

Next, for all $i\in\{2,\dots,\nbrlayers\}$,
\begin{align*}
  \abs{\neuronparE^i}~&\geq~2\log[1/\levellpr]\tj{by assumption on $\neuronparE^i$}\\
&\geq~\frac{\log[1/\levellpr]}{ 2\abs{\function^i}^2}\tj{$\abs{\function^i}\geq 1/2$ by above}\\
&\geq~\frac{1}{ 4\abs{\function^i}^2}\tj{$\levellpr<1$ by assumption}\,,
\end{align*}
which implies 
\begin{equation*}
 \frac{\abs{\neuronparE^i}}{\log\abs{\neuronparE^i}}~\geq~\sqrt{\abs{\neuronparE^i}}~\geq~\frac{1}{2\abs{\function^i}}\,.
\end{equation*}
We then find similarly as above
\begin{equation*}
\absB{\bigl(1-(\function^i)^2\bigr)\neuronparE^i}~\leq~1~~\text{for all}~i\in\{2,\dots,\nbrlayers\}\,.
\end{equation*}
Collecting the pieces yields the first claim.

The second claim then follows readily from the fact that $\tuningparameter\geq\deffitA$ and 
\begin{multline*}
 \absBB{\frac{\partial}{\partial\neuronparE^k}\Bigl(\lossW+\frac{\tuningparameter}{2}\sum_{i=1}^{\nbrlayers} (\neuronparE_i)^2\Bigr)}\\\geq~
\tuningparameter\abs{\neuronparE_i}-\absBB{\deffit\,\function^{k-1}\bigl(1-(\function^{\nbrlayers})^2\bigr)\cdots\bigl(1-(\function^k)^2\bigr)\,\prod^{\nbrlayers}_{i=k+1}\neuronparE^i}\,.
\end{multline*}
\end{proof}

\subsection{A remedy for the many-layers regime}
\label{sec:remedyml}
We finally show that rescaling can avoid the vanishing-gradient problem in the many-layers regime.
Consider a general activation function $\function\,:\,\R\to\R$ and a rescaled version of it defined by
\begin{equation}
\label{scaling}
  \begin{aligned}
  \functionscale~:~\R~&\to~\R\,;\\
   \argumentE~&\mapsto~\functionscale[\argumentE]\,\deq\, \scale\,\function[\argumentE/\scale]\,,    
  \end{aligned}
\end{equation}
where $\scale\in(0,\infty)$ is a given scale parameter.
It is well-known that $\functionscale=\function$ for relu-type functions \citep{Hebiri2020,Taheri2020};
in other words, relu-type activation is invariant under the above rescaling.

We, however, are interested in rescaled versions of sigmoid activations.
For reference, we call
\begin{align*}
  \functiontanhscale~:~\R~&\to~(-\scale,\scale)\\
   \argumentE~&\mapsto~       \functiontanhscaleF{\argumentE}\deq \scale\,\functiontanhF{\argumentE/\scale}
\end{align*}
the \emph{$\scale$-scaled \nametanh}.
More generally, we speak of \emph{$\scale$-scaled sigmoid functions},
and we call replacing the original sigmoid functions by their scaled versions the \textit{scaling trick.}
The scaling trick retains the flavor of the original \nametanh\ for small arguments ($\functiontanhscaleF{\argumentE}\approx\functiontanhF{\argumentE}$ for $\argumentE\approx 0$),
but it essentially multiplies the \nametanh\ by~\scale\ for large arguments ($\functionscaleF{\argumentE}\approx\scale\functionF{\argumentE}$ for $\abs{\argumentE}\gg 0$).

A consequence of these properties is that we can increase the derivatives---$(\functionscaleF{\argumentE})'>(\functionF{\argumentE})'$ for $\scale>1$---and, therefore, avoid  the vanishing-gradient problem in the many-layers regime.
We illustrate this feature once more in our toy model.
Generalizing \nametanh\ to $\scale$-scaled \nametanh\ leads to the partial derivatives
\begin{equation*}
 \frac{\partial}{\partial\neuronparE^i}\lossW~=
~\deffit\,\scale\,\function_{\scale}^{k-1}\bigl(1-(\function_{\scale}^{\nbrlayers})^2\bigr)\cdots\bigl(1-(\function_{\scale}^k)^2\bigr)\,\prod^{\nbrlayers}_{i=k+1}\neuronparE^i\,,
\end{equation*}
where
\begin{equation*}
  \function_{\scale}^i~\deq~\functiontanhFBB{\neuronparE^{i}\functiontanhFB{\neuronparE^{i-1}\cdots\functiontanhF{\neuronparE^1 \datainEnew/\scale}}}
\end{equation*}
for all $i\in\{1,\dots,\nbrlayers\}$ and $\function_{\scale}^0\deq1$.
This equality can be derived from the general fact that 
\begin{equation*}
  \functionscaleFBB{\functionscaleFB{\dots\functionscaleF{\argumentE}}}~=~\scale\,\functionFBB{\functionFB{\dots\functionF{\argumentE/\scale}}}~~~\text{for all}~\argumentE\in\R
\end{equation*}
irrespective of the base function~$\function$.
One can readily verify that the above derivatives equal to the ones in the previous section in the case $\scale=1$.
The following theorem now compares $\scale=1$ and $\scale>1$ in the context of vanishing gradients.
\begin{theorem}[Many-Layers Regime]
\label{res:MLR}
Assume that $ \abs{\datainEnew}~=~\abs{\neuronparE^1}~=~\cdots~=~\abs{\neuronparE^{\nbrlayers}}~=~1.5$.
Then, if $\scale=1$ (usual \nametanh\ activation),
it holds  for all $k\in\{1,\dots,\nbrlayers\}$ that
\begin{equation*}
\absBB{\frac{\partial}{\partial\neuronparE^k}\lossW}~\leq~\deffitA\,\frac{1}{2^{\nbrlayers-k+1}}\,,
\end{equation*}
while for increasing $\scale$ (cf.~scaling trick)
\begin{equation*}
\lim_{\scale\to\infty}\absBB{\frac{\partial}{\partial\neuronparE^k}\lossW}~=~\deffitA\,(1.5)^{\nbrlayers-1}\,\abs{\datainEnew}\,.
\end{equation*}
\end{theorem}
\noindent
The result can be readily generalized to other, not necessarily equal, values of $\datainEnew,\neuronparE^1,\dots,\neuronparE^{\nbrlayers}$,
but the given choice already captures the key aspects.
Indeed,
we see that the gradients in the standard case can vanish with  increasing depth (see the factor $1/2^{\nbrlayers-k+1}$),
while this behavior is avoided with larger~$\scale$ 
Thus, the theorem illustrates that deep learning can suffer from vanishing gradients in the many-layers regime, and it suggests the scaling trick as a remedy for it.

The factor $(1.5)^{\nbrlayers-1}$ in the second claim is just an artifact of taking the limit $\scale\to\infty$;
in particular,
the scaling trick does not lead to exploding for any practical choice of~$\scale$.
We confirm this fact,
as well as the fact that regularization does not lead to exploding gradients either,
in the experiments section.
In contrast,
 ``naive scaling'' $\functionF{\argumentE}\to\scale\,\functionF{\argumentE}$,
which one could also use to prevent  vanishing gradients,
clearly entails the risk of exploding gradients:
\begin{equation*}
\lim_{\scale\to\infty}\absBB{ \frac{\partial}{\partial\neuronparE^k}\lossW }~=~\infty\,. 
\end{equation*}

The parameter~$\scale$ can be learned from data---similarly as the parameters of other parametrized activation functions~\citep[Section~2.3.4]{Lederer2021}.
But it turns out that the fixed choice $\scale=2$ is sufficient in practice (see the experimental results),
which makes the implementation of the scaling trick very easy and efficient.

Let us finally compare our scaling approach to  batch normalization \citep{Ioffe2015}.
Batch normalization attempts to ``whiten'' the data within each gradient step,
 that is, to keep the distribution of the inputs constant during training. 
It is related to our approach in that their scaling of the data can also be seen as a scaling of the parameters.
But batch normalization involves (usually many) new parameters and considerable computational overhead,
while our scaling works with a single, predefined parameter and no other changes to the optimization routine.
Moreover, 
batch normalization alters the network (essentially the activations) drastically,
while scaling retains the flavor of the original network.
Hence,
our scaling trick could be interpreted as a simpler, more efficient, and less invasive alternative to  batch normalization.

\begin{proof}[Proof of Theorem~\ref{res:MLR}]
Consider first usual \nametanh\ activation.
It then holds that $\abs{\function^{k-1}}\leq 1$ by definition of the \nametanh\ function.
Next,
under the stated assumptions, it holds that 
\begin{equation*}
  \abs{\datainEnew}~\geq~1.5\,\functiontanhF{1.5}~~~~~\text{and}~~~~~\abs{\neuronparE^1},\dots,\abs{\neuronparE^{\nbrlayers}}~\geq~\frac{1.3}{\functiontanhF{1.5}}\,.
\end{equation*}
Hence, $\abs{\function^i}\geq\functiontanhF{1.3}$ for all $i\in\{0,\dots,\nbrlayers\}$ and, therefore, 
\begin{equation*}
  \absB{1-(\function^i)^2}~\leq~\absB{1-\bigl(\functiontanhF{1.3}\bigr)^2}~\leq~\frac{1}{2\cdot 1.5}\,.
\end{equation*}
And finally, $\abs{\neuronparE^i}\leq 1.5$ by assumption.
Putting these results into our earlier calculation of the gradient yields 
\begin{equation*}
\absBB{\frac{\partial}{\partial\neuronparE^k}\lossW}~\leq~\deffitA\cdot1\cdot\frac{1}{2^{\nbrlayers-k+1}}\,,
\end{equation*}
which implies the first claim.

Consider now  $\scale$-scaled activation.
The approximate linearity of the \nametanh\ activation at the origin yields
\begin{equation*}
  \lim_{\scale\to\infty}\scale\, \function_{\scale}^i~=~\prod_{m=1}^i\neuronparE^m\datainEnew\,.
\end{equation*}
This equality, in turn, implies $\function_{\scale}^i\to0$ for $\scale\to\infty$ and, therefore, $1-(\function_{\scale}^i)^2\to1$.
Collecting the pieces yields
\begin{equation*}
\lim_{\scale\to\infty}\absBB{\frac{\partial}{\partial\neuronparE^k}\lossW}~=~ \deffitA\,\absBBBB{\prod_{m=1}^{k-1}\neuronparE^m\datainEnew}\cdot\absBBBB{\prod_{i=k+1}^{\nbrlayers}\neuronparE^i}\,,
\end{equation*}
which implies the second claim.
\end{proof}
\noindent

\section{Empirical Support}
\label{sec:empirical}

\newcommand{\figurecombinedA}[1]{
\begin{figure*}[#1]
	\centering
	\caption{
accuracies (top) and summary statistics of the coordinate values of the gradients (bottom)  in the combined regime
\unc{a better label for the blue line is ``unregularized_unscaled'' and for the orange line ``regularized_scaled''}
}
	\includegraphics[width=\textwidth]{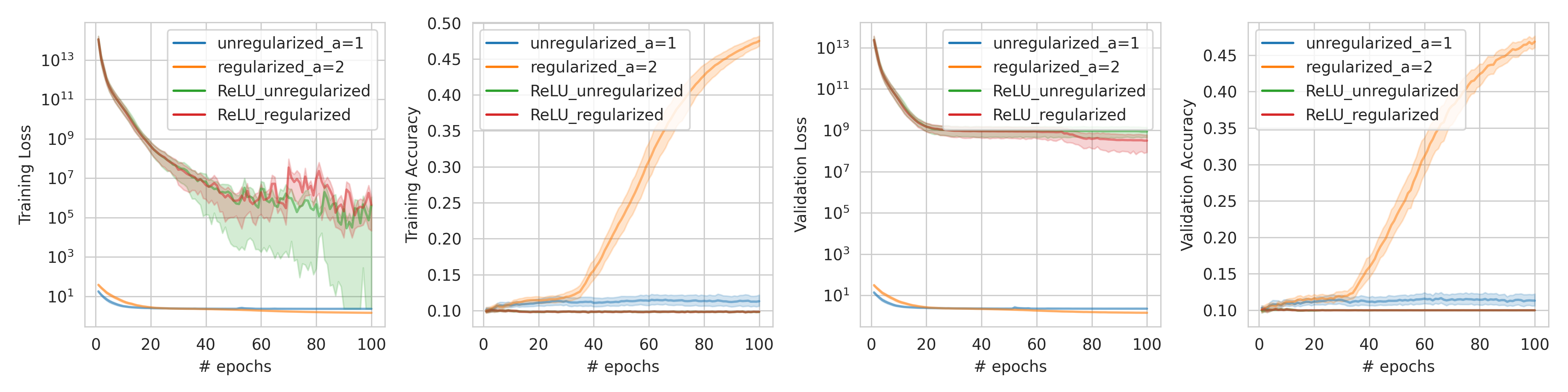}
	\label{fig:Model Combined hist}
	\includegraphics[width=\textwidth]{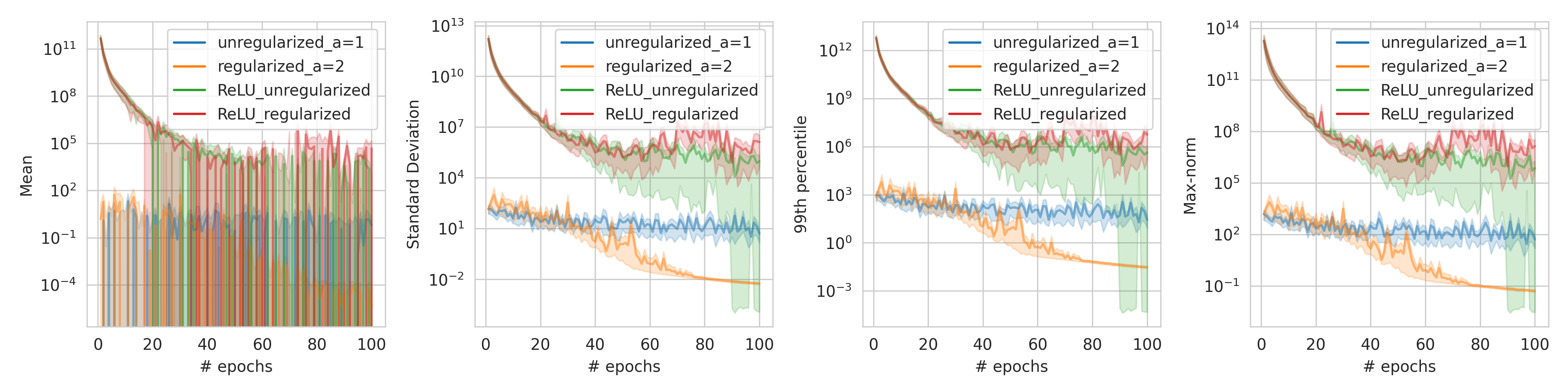}
\end{figure*}}

We now illustrate that the two remedies can indeed avoid vanishing gradients in
practice. 
To disentangle the effects,
we consider the large-parameters regime and the many-layers regime
individually and combined.

The data under consideration is CIFAR-10 \citep{cifar10}. 
The deep-learning ecosystem  is  \texttt{TensorFlow} in \texttt{Python}; 
the optimization algorithm is \texttt{Adam} with learning rate $0.001$.
Some details on the computations and further empirical results are deferred to the Appendix.

\subsection{Large-parameters regime}

We first demonstrate that regularization can avoid vanishing gradients in the
large-parameters regime. The architecture of the neural-network class under
consideration is given in Table~\ref{table:Model LP} (all these tables are moved to the Appendix). 
The initial parameters are sampled independently and
identically from a uniform distribution supported on the interval $[-10,10)$.
In practice, parameters are usually initialized with smaller values,
but our broad range of possible values 1.~allows us to illustrate the fact that regularization works even for very large, and very different, parameter values and 2.~can become realistic, irrespective of the initialization, during training anyways.
The activation function is $\nametanh$ except for the pooling and output layers,
which use linear and soft-max activation,
respectively.

We compare unregularized training ($\tuningparameter=0$) and $\ell_2$-regularized training ($\tuningparameter=0.01$)---cf.~Section~\ref{sec:remedylpr}.
The pipeline is run 25 times, with a different seed each time.
The resulting accuracy and gradient histories are displayed in Figure~\ref{fig:Model LP hist}.

\begin{figure*}
	\centering
	\caption{
Accuracies (top) and summary statistics of the coordinate values of the gradients (bottom) in the large-parameters regime.
(Here and in the following,  solid lines denote the averages over the runs and shaded areas the corresponding interquartile ranges.)
}
	\includegraphics[width=\textwidth]{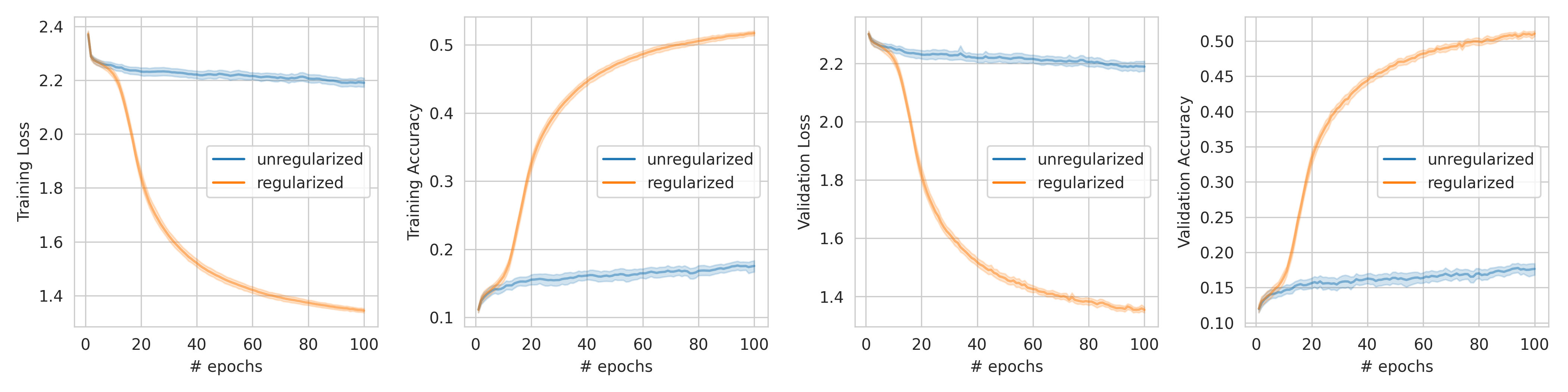}
	\label{fig:Model LP hist}
	\includegraphics[width=\textwidth]{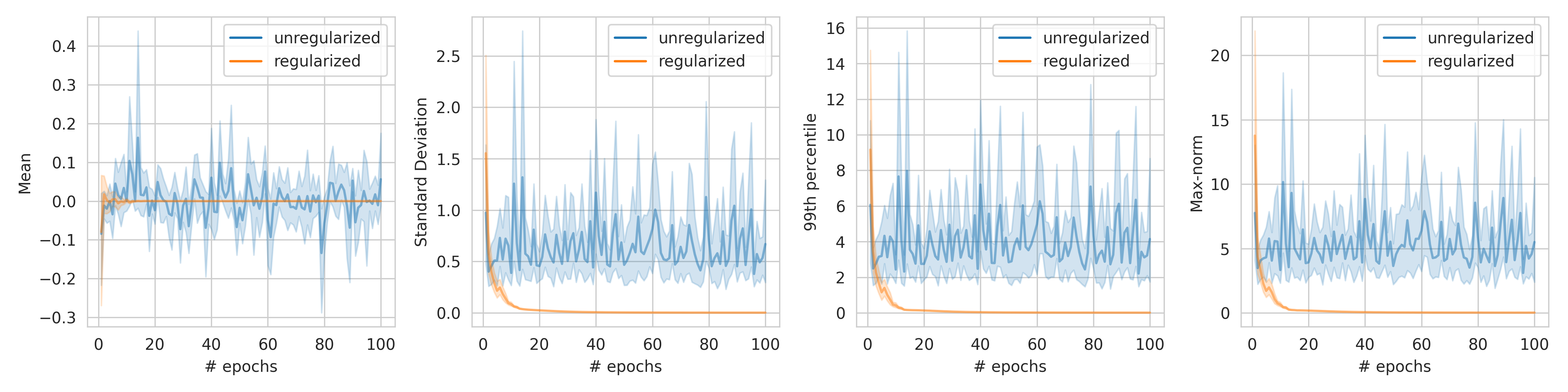}
\end{figure*}

The top panel of Figure~\ref{fig:Model LP hist} shows that unregularized and regularized training perform very similarly initially,
but after about ten epochs,
the regularization improves the accuracies considerably.
The bottom panel of that figure highlights the reason for this improvement:
regularization leads to larger gradients in the beginning,
which allows the training to ``escape'' the problematic regions eventually.
Figure~\ref{fig:Model LP hist} also demonstrates the fact that regularization does not lead to exploding gradients either:
indeed, the gradients become smaller as the training algorithm reaches more favorable regions in the parameter space.

We conclude that regularization can indeed avoid the vanishing-gradient problem (without leading to an exploding-gradient problem) in the large-parameters regime.

\unc{
I am also adding some suggestions in orange. 
These suggestions are uncritical: 
you can delay working on them,
or even disregard them.
}

\unc{
The rule is to capitalize the first letter (if not a symbol) of the axes labels: 
for example, ``Training loss'' rather than ``Training Loss.''

lw: Do you mean ``Training loss'' instead of ``Training Loss''? I treat the axis
label as a title so I capitalised both. If you want I can change them.\\
jl: Yes, sorry. I corrected my example.
I would suggest you change it, but it definitely very low priority...
}

\subsection{Many-Layers regime}
We now demonstrate that the scaling trick can avoid vanishing gradients in the many-layers regime.
The network class is specified in  Table~\ref{table:Model DL}. 
The table shows that the many-layers regime is modeled, not surprisingly, by stacking many layers.
The parameters are initialized based on a centered normal distribution with standard deviation $0.05$.
The activations are the same as above.
We compare unscaled training (using the standard \nametanh\ activation) and scaled training (using $\scale$-scaled \nametanh \ activation with $\scale=2$).
Again, we average the results over $25$ runs with different seeds.
The resulting accuracy and gradient histories are displayed in Figure~\ref{fig:Model DL hist}.

We observe a very similar pattern as in the large-parameters regime:
the larger gradients in the beginning (see the bottom panel of the figure) speed up parameter training considerably (see the top panel).
The main difference to above is that the improvements kick in immediately (indicating that there are no plateaux from which the algorithm needs to ``escape'').

We conclude that the scaling trick can avoid the vanishing-gradient problem, again without leading to an exploding-gradient problem, in the many-layers regime.

\begin{figure*}
	\centering
	\caption{
accuracies (top) and summary statistics of the coordinate values of the gradients (bottom) in the many-layers regime
\unc{we might need to adjust those---cf slack...}
}
	\includegraphics[width=\textwidth]{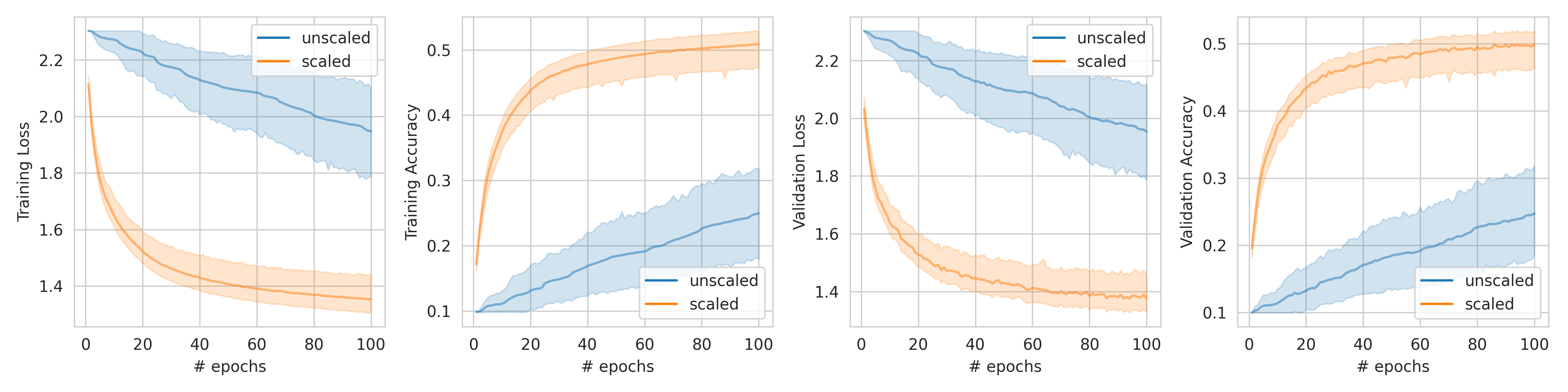}
	\label{fig:Model DL hist}
	\includegraphics[width=\textwidth]{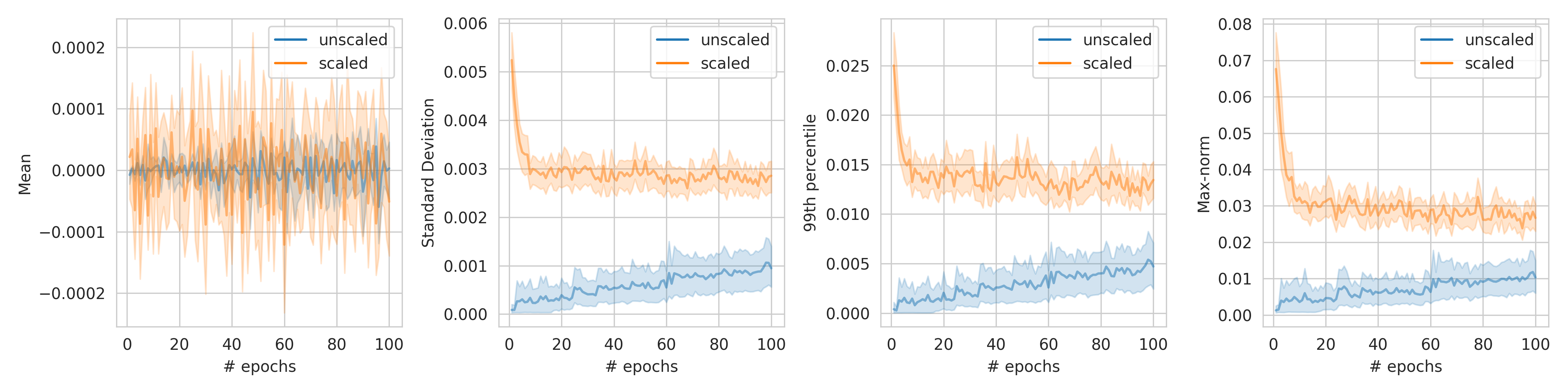}
\end{figure*}

\figurecombinedA{t}
\subsection{Combined regime and comparison with \namerelu\ activation}

We finally demonstrate that regularization and scaling can be combined.
The outline of the model is given in Table~\ref{table:Model Combined}. The parameters are initialized based on a uniform distribution supported on $[-10,10)$.
The activations and the parameters for the regularization and the scaling tricks are the same as above.
We compare the standard pipeline with the one that includes regularization and scaling,
and we also consider \nametanh\ replaced by \namerelu.
The results over $25$ runs with different seeds are displayed in Figure~\ref{fig:Model Combined hist}  (since scaling does not affect relu-type functions, only regularization is mentioned in the \namerelu\ case).
The results here corroborate the results above and show that the two remedies can indeed be combined.
An additional observation is that \namerelu, in this specific example, is not performing well.

Table~\ref{table:Comparison of runtimes} in the Appendix contains the runtimes of one model from the combined regime
 executed on the same, regular graphics card.
It shows that our two remedies do not  involve significant computational overhead.

We conclude that our remedies can render \nametanh\ robust with respect to the vanishing-gradient problem.
Our experiments are not designed to give an overall comparison of different types of activation functions, such as \nametanh\ and \namerelu, or even to find the ``best'' activation function;
in fact, which activation function is the ``best'' one surely depends on the specific settings (architecture, data, parameters, ...), the goal (computational efficiency, accuracy, ...), and so forth.
But what our study is able to show is that sigmoid-type activation is a serious contender.

\section{Discussion}
\label{sec:discussion}

The paper provides a fresh look at the vanishing-gradient problem in the context of sigmoid-type activation.
In particular,
the paper shows that two simple remedies,
regularization and rescaling,
can avoid the vanishing-gradient problem.
This observation suggests that sigmoid-type activation is still a contender in practice.

There are also efforts to study the vanishing- (and exploding-)gradient problems from the perspective of mean-field theory \citep{RWDL,Hanin2018}.
These works currently focus on how network widths and depths relate to vanishing gradients, 
but it would be interesting to see if their approaches can also yield further theoretical insights into our findings.

For the sake of a clear and concise delivery of the main ideas, we have restricted ourselves to feedforward networks.
But we have every reason to believe that our concepts apply much more generally,
such as to recurrent neural networks (RNNs) \citep{kalchbrenner2013recurrent,mikolov2012statistical}, for example.
Similarly,
the numerical setup focuses on \nametanh\ activation,
which is currently the most popular sigmoid-type activation,
but our remedies also apply to other activation functions,
such as \namesoft\ activation---cf.~\citep[Section~3]{Lederer2021}---for example.
And finally,
other types of regularizations, such as $\ell_1$-regularization \citep{Alvarez16,Feng17,Kim16}, are expected to avoid vanishing gradients similarly as $\ell_2$-regularization.

Many questions, however, remain open. 
In particular, our paper does not yet give a final answer to the question of which activation function to use in a specific setting.
But our paper illustrates that an entire class of activation functions might have been underrated,
and,  on a more general level,
it illustrates that conceptual and mathematical research---rather than anecdotal evidence or hearsay---can lead to more informed design choices in practice.


\ifarXiv
\subsection*{Acknowledgements}
We thank Xingtu Liu, Pegah Golestaneh, Mike Laszkiewicz, and Nils M\"uller for helpful input.
\else

\fi

\bibliography{Sections/Bibliography}



\newpage
\appendix
 \section{Experimental Details}
Here, we provide the tables about the architectures as mentioned in the experimental section.

\begin{table}[h]
	\caption{neural-network architecture for the large-parameters regime}
	\centering
	\label{table:Model LP}
        \scriptsize
	\begin{tabular}{c|c|c}
\toprule
		Layer & Output dimensions & Filters \\
		\midrule
		Input       & $(32,32,3)  $ & \\
		3-by-3 convolution & $(32,32,3)  $ & 36 \\
		3-by-3 convolution & $(32,32,3)  $ & 36 \\
		Max pooling     & $(16,16,3)  $ & 0 \\
		3-by-3 convolution & $(16,16,3)  $ & 36 \\
		3-by-3 convolution & $(16,16,3)  $ & 36 \\
		Max pooling     & $(8,8,3)    $ & 0 \\
		3-by-3 convolution & $(8,8,3)    $ & 36 \\
		3-by-3 convolution & $(8,8,3)    $ & 36 \\
		Max pooling     & $(4,4,3)    $ & 0 \\
		Dense       & $(8)        $ & 392 \\
		Dense       & $(10)       $ & 90\\
\bottomrule
	\end{tabular}
\end{table}

\begin{table}[h]
	\caption{neural-network architecture for the many-layers regime}
	\centering
	\label{table:Model DL}
        \scriptsize
	\begin{tabular}{c|c|c}
\toprule
		Layer & Output dimensions & Filters \\
		\midrule
		Input       & $(32,32,3)  $ & \\
		3-by-3 convolution & $(32,32,3)  $ & 3 \\
		3-by-3 convolution & $(32,32,3)  $ & 3 \\
		3-by-3 convolution & $(32,32,3)  $ & 3 \\
		Max pooling     & $(16,16,3)  $ & 3 \\
		3-by-3 convolution & $(16,16,3)  $ & 3 \\
		3-by-3 convolution & $(16,16,3)  $ & 3 \\
		3-by-3 convolution & $(16,16,3)  $ & 3 \\
		Max pooling     & $(8,8,3)    $ & 3 \\
		3-by-3 convolution & $(8,8,3)    $ & 3 \\
		3-by-3 convolution & $(8,8,3)    $ & 3 \\
		3-by-3 convolution & $(8,8,3)    $ & 3 \\
		Max pooling     & $(4,4,3)    $ & 3 \\
		Dense       & $(8)        $ & 384 \\
		Dense       & $(8)        $ & 64 \\
		Dense       & $(8)        $ & 64 \\
		Dense       & $(8)        $ & 64 \\
		Dense       & $(10)       $ & 80\\
\bottomrule
	\end{tabular}
\end{table}

\begin{table}[h]
	\caption{neural-network architecture for the combined regime}
	\centering
	\label{table:Model Combined}
        \scriptsize
	\begin{tabular}{c|c|c}
\toprule
		Layer & Output dimensions & Filters \\
		\midrule
		Input       & $(32,32,3)  $ & \\
		3-by-3 convolution & $(32,32,3)  $ & 3 \\
		3-by-3 convolution & $(32,32,3)  $ & 3 \\
		Max pooling     & $(16,16,3)  $ & 3 \\
		3-by-3 convolution & $(16,16,3)  $ & 3 \\
		3-by-3 convolution & $(16,16,3)  $ & 3 \\
		Max pooling     & $(8,8,3)    $ & 3 \\
		3-by-3 convolution & $(8,8,3)    $ & 3 \\
		3-by-3 convolution & $(8,8,3)    $ & 3 \\
		Max pooling     & $(4,4,3)    $ & 3 \\
		Dense       & $(8)        $ & 384 \\
		Dense       & $(8)        $ & 64 \\
		Dense       & $(8)        $ & 64 \\
		Dense       & $(10)       $ & 80\\
\bottomrule
	\end{tabular}
\end{table}
\newpage
\section{Runtimes}
Here, we provide the table about the runtime as mentioned in the experimental section.
\begin{table}[h]
	\caption{comparison of runtimes}
\small
	\centering
	\label{table:Comparison of runtimes}
	\begin{tabular}{l|r}
		\toprule
		Model & Runtime (s) \\
		\midrule
		\nametanh\ unregularized and unscaled & $1057.8$ \\
		\nametanh\ regularized and scaled & $1385.2$ \\
		\namerelu\ unregularized & $986.8$ \\
		\namerelu\ regularized & $1339.1$ \\
		\bottomrule
	\end{tabular}
\end{table}

\section{Other Sigmoid-Type Activations}
In the main part of the paper,
we have focused on \nametanh\ activation.
But we have also mentioned that our results can be transferred readily to other sigmoid-type activation.
For illustration,
we provide an analog of Figure~\ref{fig:Model Combined hist} for the case of \nametanh\ replaced by \namelog.
We obtain virtually the same results as before.

\begin{figure}[h]
	\centering
	\caption{
		accuracies (top) and summary statistics of the coordinate values of the gradients (bottom)  in the combined regime with \nametanh\ replaced by \namelog\  activation}
	\includegraphics[width=\textwidth]{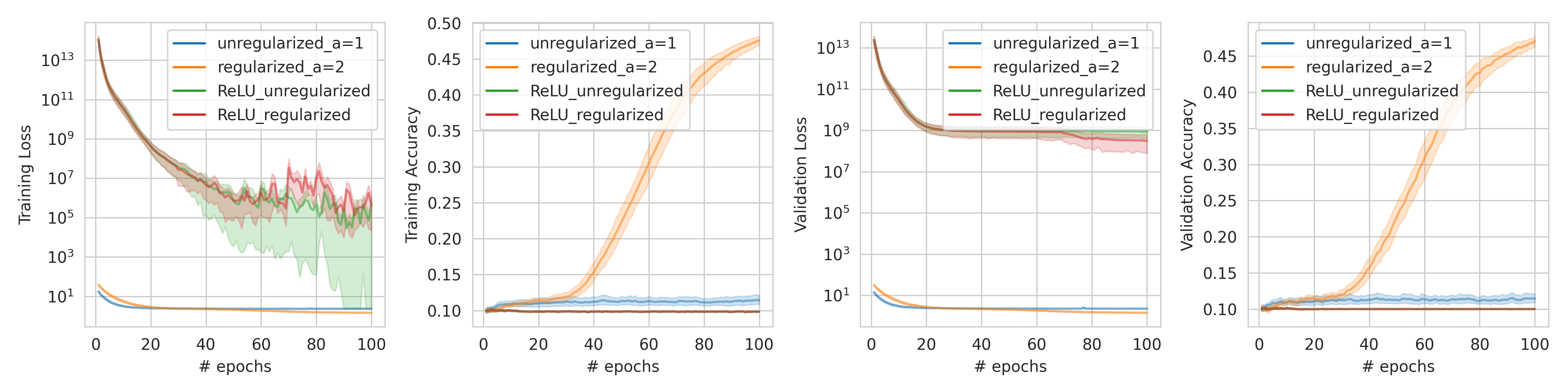}
	\includegraphics[width=\textwidth]{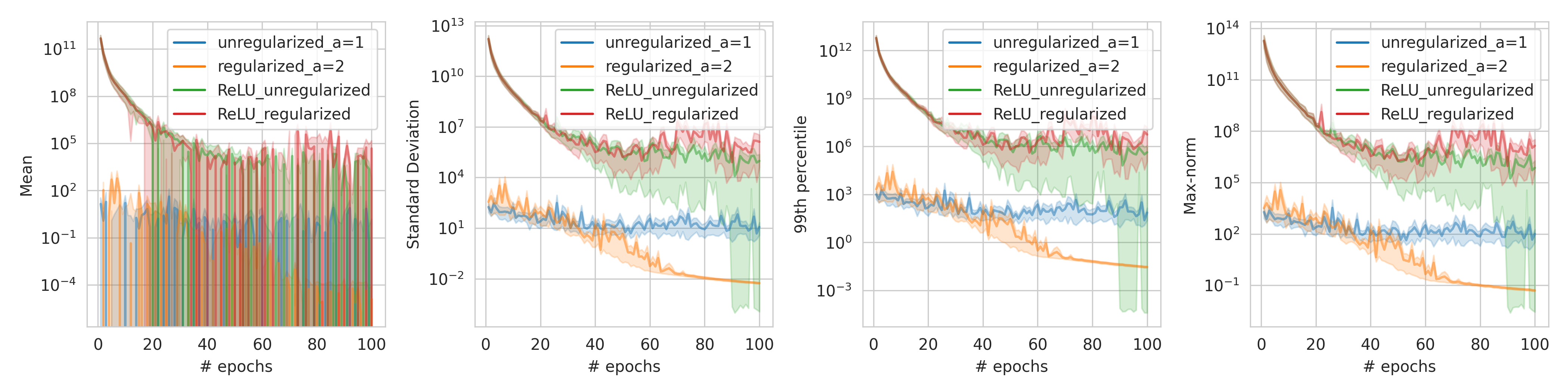}
\label{fig:logsig}
\end{figure}


\end{document}